\documentclass[]{article}
\usepackage{proceed2e}

\usepackage{times}

\usepackage{times}
\usepackage{helvet}
\usepackage{courier}
\usepackage{bm}
\usepackage{algorithm}
\usepackage{algorithmic}
\usepackage{graphicx}
\usepackage{subfigure}
\usepackage{multirow}%
\usepackage{amsmath, amssymb}
\usepackage{color}
\usepackage{mathrsfs}
\usepackage{amsthm}
\usepackage{epstopdf}
\usepackage{wrapfig}
\usepackage{picinpar}
\usepackage{url}
\usepackage{threeparttable}

\DeclareMathOperator*{\argmin}{arg\,min}

\def\diag{\mbox{diag}}
\def\rank{\mbox{rank}}
\def\grad{\mbox{\text{grad}}}

\def\sgn{\mbox{sgn}}
\def\tr{\mbox{tr}}

\def\bzeta{\mbox{{\boldmath $\zeta$}}}

\def\bsigma{\mbox{{\boldmath $\sigma$}}}
\def\bzeta{\mbox{{\boldmath $\zeta$}}}

\def\bxi{\mbox{{\boldmath $\xi$}}}

\def\bvarsigma{\mbox{{\boldmath $\varsigma$}}}
\def\bXi{\mbox{{\boldmath $\Xi$}}}

\def\mA{{\mathcal A}}
\def\mB{{\mathcal B}}

\def\mM{{\mathcal M}}

\def\mP{{\mathcal P}}

\def\mU{{\mathcal U}}
\def\mV{{\mathcal V}}

\def\0{{\bf 0}}
\def\1{{\bf 1}}

\def\bA{{\bf A}}
\def\bB{{\bf B}}

\def\bD{{\bf D}}
\def\bE{{\bf E}}

\def\bG{{\bf G}}
\def\bH{{\bf H}}
\def\bI{{\bf I}}

\def\bM{{\bf M}}

\def\bU{{\bf U}}
\def\bV{{\bf V}}

\def\bX{{\bf X}}
\def\bY{{\bf Y}}
\def\bZ{{\bf{Z}}}

\def\bd{{\bf d}}

\def\bn{{\bf n}}

\def\bv{{\bf v}}

\def\by{{\bf y}}

\def\mmR{{\mathbb R}}
\def\mMLr{{\mM_{\leq r}}}

\def\by{\bar{\y}}

\def\trsp{{\sf T}}

\def\mRMD{{\mathrm{D}}}
\def\mRMD{{\mathrm{D}}}

\def\shijie{\textcolor{black}}
\def\shijienew{\textcolor{black}}

\def \twoone{\shijie{2,1}}
\def \ieShijie{\shijie{\emph{i.e.},}}
\def \egShijie{\shijie{\emph{e.g.},}}

\def\citep{\cite}
\def\citet{\cite}

\newtheorem{thm}{Theorem}
\newtheorem{prop}{Proposition}
\newtheorem{lemma}{Lemma}
\newtheorem{remark}{Remark}

\newenvironment{myblist}{\vspace{-.3cm}
\begin{itemize}
                        \setlength{\itemsep}{0pt}
                        \setlength{\parskip}{0pt}
                        \setlength{\parsep}{0pt}}{\end{itemize}
                        \vspace{-.2cm}}

\title{Scalable Trace-norm Minimization \\by Subspace Pursuit Proximal Riemannian Gradient}

\author{ {\bf Mingkui Tan\dag\thanks{\dag School of Computer Science, The University of Adelaide; \ddag School of Computer Engineering, Nanyang Technological University; \dag\dag  Computer Science in the School of Computing and Mathematics, Charles Sturt University.}} \\
\And
{\bf Shijie Xiao\ddag}  \\
\And
{\bf Junbin Gao\dag\dag}   \\
\And
{\bf Dong Xu\ddag}   \\
\And
{\bf Anton Van Den Hengel \dag}   \\
\And
{\bf Qinfeng Shi \dag}   \\
}

\begin{document}

\maketitle

\begin{abstract}
Trace-norm regularization plays a vital role in many learning tasks, such as low-rank matrix recovery (MR), and low-rank representation (LRR). Solving this problem directly can be computationally expensive due to the unknown rank of variables or large-rank singular value decompositions (SVDs). To address this, we propose a proximal Riemannian gradient (PRG) scheme which can efficiently solve trace-norm regularized problems defined on real-algebraic variety $\mMLr$ of real matrices of rank at most $r$. Based on PRG, we further present a simple and novel subspace pursuit (SP) paradigm for general trace-norm regularized problems without the explicit rank constraint $\mMLr$. The proposed paradigm is very scalable by avoiding large-rank SVDs. Empirical studies on several tasks, such as  matrix completion and LRR based subspace clustering, demonstrate the superiority of the proposed paradigms over existing methods.


\end{abstract}
\section{Introduction}
\label{introduction}

Trace-norm regularization plays a vital role in various areas, such as machine learning~\cite{yuan2007dimension,argyriou2008convex}, data mining~\cite{Salakhutdinov2010}, computer vision and image processing~\cite{LRRPAMI,peng2012rasl,wang2014robust}. Most trace-norm based problems can be formulated into the following general formulation~\cite{lin2011linearized}:
\begin{eqnarray}\label{eq:general_form}
\min_{\bX,\bE} ~||\bX||_* + \lambda\Upsilon(\bE), ~\mathrm{s.t.} ~~ \mA(\bX) + \mB(\bE) = \bD,
\end{eqnarray}
where $\lambda$ is a regularization parameter, $||\bX||_*$ is the trace-norm (also known as the nuclear-norm) of a matrix $\bX\in\mmR^{m\times n}$, both $\mA$ and $\mB$ are linear operators depending on specific applications~\cite{lin2011linearized}, $\bD$ denotes data or observations, $\bE$ can be considered as an error term and %
\shijie{$\Upsilon(\bE)$ is a regularizer on $\bE$ which is possibly non-smooth.} %
The trace norm $\|\bX\|_*$ is the tightest convex lower bound to the rank function $\rank(\bX)$~\cite{Recht2010}, and the minimization of (\ref{eq:general_form}) encourages the variable $\bX$ to be low-rank~\cite{fazel2002matrix,Hazan2008,Cand2009Exact}.
Among various trace-norm based problems, the low-rank matrix recovery (MR)~\cite{Cand2010Tight}, %
\shijie{and low-rank representation~(LRR)~\cite{LRRPAMI},} %
have gained particular interest in the last decade.

MR \shijie{\cite{Cand2010Tight}}~%
seeks to recover a low-rank matrix $\bX$ from  partial observations that are %
recorded
in a vector $\bd \in \mmR^l$, where $l\!\ll\! mn$. If there are no outliers in the observations, one can recover $\bX$ with high probability by solving the following problem~\cite{Cand2010Tight,keshavan2010matrix}:
%
\begin{eqnarray}\label{eq:mc}
\min_{\bX} \! \|\bX\|_*,~~\textrm{s.t.}~~\mA(\bX) =\bd
\shijie{,}
\end{eqnarray}
%
\shijie{which} can be deemed as a simplified version of formulation (\ref{eq:general_form}). %
MR %
has been successfully applied in many tasks such as matrix completion filtering~\cite{Salakhutdinov2010,recht2013parallel}. However, the recovery performance by solving problem (\ref{eq:mc}) might be seriously degraded if the observations contain severe outliers~\cite{chen2011robust,chen2013low}. To improve the robustness, we may introduce an additional variable $\bE$ into the constraint $\mA(\bX) =\bd$ as in (\ref{eq:general_form}), and regularize it using $\ell_1$-norm regularization (\ieShijie~$||\bE||_1$) or $\ell_{\twoone}$-norm regularization (\ieShijie~$||\bE||_{\twoone}$)~\cite{chen2011robust,chen2013low}.
\shijie{LRR~\cite{LRRPAMI}} %
seeks to find a %
low-rank representation \shijie{$\bX \in \mmR^{n\times n}$} %
of given data $\bD \in \mmR^{m\times n}$ by solving an optimization problem
of the following form:
 $\min_{\bX,\bE} {\lambda}\|\bX\|_{*} +  \|\bE\|_{2,1}~~
\mbox{s.t.}~~ \bD\bX + \bE  = \bD, $ %
where \shijie{$\bD$} denotes the \shijie{given data with $n$ samples}, %
and $\|\bE\|_{2,1}$ encourages the representation error $\bE$ to be column-wise sparse. LRR has been widely applied in many real-world tasks such as motion segmentation and face clustering~\cite{LRRPAMI,lin2011linearized,vidal2010tutorial}.



Many algorithms have been proposed to solve
trace-norm regularized problems~\cite{Hazan2008,Laue2012,Jaggi2010,Zhouchen2010ALM,Toh2010APG,zhang2012accelerated,Cai2010SVT,sra2011optimization,shen2012augmented}, but most focus on solving problem (\ref{eq:mc}), such as the singular value thresholding ({SVT})~\citep{Cai2010SVT}, augmented Lagrangian method ({ALM}) and alternating
direction method (ADM)~\cite{Zhouchen2010ALM,sra2011optimization,shen2012augmented}.
Unlike these methods, some researchers proposed to solve an equivalent problem to problem (\ref{eq:mc})~\cite{Toh2010APG,ji2009accelerated}: $\min_{\bX} \|\bX\|_* + \frac{\gamma}{2}\|\mA(\bX)-\bd\|_2^2,$ where $\gamma$ is a regularization parameter. This problem is known as the \emph{{matrix lasso}}, and can be addressed by proximal gradient (PG) or accelerated proximal gradient (APG)~\cite{Combettes2005,Toh2010APG,ji2009accelerated}.

The optimization of problem (\ref{eq:general_form}) is more challenging due to the additional variable $\bE$. By minimizing $\bX$ and $\bE$ alternatively, the aforementioned methods (\emph{e.g.}, ADM and PG) have been extended to solve this problem~\cite{Zhouchen2010ALM,lin2011linearized}.

The above methods have shown great success in
practice~\cite{Toh2010APG,Zhouchen2010ALM}. However, the
optimization usually involves repetitive SVDs due to the SVT operation, making them inefficient on large-scale problems~\cite{Bonnabel2011,Vandereycken2013}. Using homotopy strategies and applying rank prediction techniques may accelerate the convergence speed with truncated SVDs~\cite{Toh2010APG,Zhouchen2010ALM,lin2011linearized}. %
However, the rank prediction could be non-trivial in general, %
and large-rank SVDs is still inevitable if %
the optimal solution has a large rank.
%





To develop more scalable algorithms, some researchers have tackled a version of the problem in which it is assumed that the rank of $\bX$ is known, \egShijie~$\!\!\rank(\bX) \!\!=\!\! {r}$, and thus proposed to solve a variational form of problem (\ref{eq:mc})~\cite{Recht2010,Jaggi2010}:
$\min_{\bG,\bH}  \|{\bG}\|_F^2 + \|{\bH}\|_F^2, ~\textrm{s.t.} ~
\mA(\bG\bH^{\top}) = \bD,$ %
where ${\bG}\!\in\!\mmR^{m\times {r}}$ and ${\bH}\!\in \!\mmR^{n\times {r}}$.
Many methods have been developed to address this problem, such as gradient based methods~\cite{Recht2010,Jaggi2010} and stochastic gradient methods~\cite{Wen2012,recht2013parallel,Avron2012}. However, these methods may still suffer from slow convergence speeds~\cite{Mishra2012,Vandereycken2013}.

Recently, fixed-rank methods  by exploiting the smooth geometry of matrices on fixed-rank manifolds have shown great advantages in computation for solving matrix recovery problems~\cite{Bonnabel2011,Boumal2012,absil2014two}, such as the low-rank geometric conjugate gradient method ({LRGeomCG})~\cite{Vandereycken2013}, %
the quotient geometric matrix completion method ({qGeomMC})~\cite{Mishra2012}, %
and the method of scaled gradients on Grassmann manifolds for matrix completion ({ScGrassMC})~\cite{Ngo2012}.
However, these methods can only deal with smooth objectives. Moreover, the rank parameter $r$ is usually unknown
in practice, and nontrivial to discover.




Motivated by the superiority of Riemannian gradient-based methods on low-rank matrix recovery problems~\cite{Vandereycken2013}, in this paper, we exploit  classical proximal gradient methods and  geometries
of the
real-algebraic variety $\mMLr$ to address the problem in (\ref{eq:general_form}).  The main contributions of this paper are as follows:

\begin{myblist}
\item
We propose a proximal Riemannian gradient (PRG) scheme to address trace-norm regularized problems with explicit rank constraint $\rank(\bX)\leq r$. %
By exploiting geometries on $\mMLr$, PRG avoids repetitive large-scale SVDs of classical proximal methods, making it more scalable.



\item
To address general trace-norm regularized problems in (1), we present a simple and novel active subspace framework which incorporates PRG as a slave solver. This framework does not require the prior knowledge of $r$ and large-rank SVDs, and can even accelerate the convergence speed of PRG.
\end{myblist}


\section{Notations and Preliminaries}\label{sec:preliminary}
Let the superscript ${}^{\trsp}$ denote the
transpose of a vector/matrix, $\textbf{0}$ be a vector/matrix with all
zeros,  $\diag(\bv)$ be a diagonal matrix with diagonal
elements equal to $\bv$, ${\langle\bA,\bB\rangle} = \tr(\bA\bB^{\trsp})$
be the inner product of $\bA$ and $\bB$, and $\|\bv\|_p$ be the $\ell_p$-norm of a vector $\bv$.  Let $\mA$ be a linear operator with ${\mA^*}$ being its adjoint operator.  The operator $\max(\bsigma,\bv)$ operates on each dimension of $\bsigma$.
Let $\bX={\bU} \diag({\bsigma})\bV^{\trsp}$ be the SVD of $\bX\in \mmR^{m\times n}$. The nuclear norm of $\bX$ is defined as $\|\bX\|_* =\|\bsigma\|_1
= \sum_i |\sigma_i|$ and the Frobenius norm of $\bX$ is defined as
$\|\bX\|_F =  \|\bsigma\|_2$. Lastly, for any convex function $\Omega(\bX)$, let $\partial \Omega(\bX)$ denote its subdifferential at $\bX$.

We
now introduce some of the basic notions of the geometry
of fixed-rank matrices and matrix varieties as follows.

\textbf{Geometries of Fixed-rank Matrices}. The fixed rank-$r$ matrices lie on a smooth submanifold defined below
%
\begin{eqnarray}
\mM_{r} &=& \{\bX\in \mmR^{m\times n}: \rank(\bX) = r\} \nonumber \\
&=& \{\bU\diag(\bsigma)\bV^{\trsp}: \bU \in \textrm{St}_{r}^{m}, \bV \in
\textrm{St}_{r}^{n}, ||\bsigma||_0 = r\}, \nonumber
\end{eqnarray}
%
\noindent where $\textrm{St}_{r}^{m} = \{\bU \in \mmR^{m\times r}:
\bU^{\trsp}\bU = \bI\}$ denotes the Stiefel manifold of $m\times r$ real and
orthonormal matrices, and the entries in $\bsigma$ are in descending order~\cite{Vandereycken2013}. Moreover, the tangent space $T_{\bX}\mM_{r}$
at $\bX$ is given by

\vspace{-0.20in}
\begin{small}
\begin{eqnarray}\label{eq:tangent_space}
T_{\bX}\mM_{r} \!\!\!\!\!&=&\!\!\!\!\! \{\bU\bM\bV^{\trsp} \!\!+\! \bU_p\bV^{\trsp} \!\!+\!
\bU\bV_p^{\trsp}: \!\bM \in \mmR^{r\times r},\! \bU_p \in \mmR^{m\times r}, \nonumber \\
&& \bU_p^{\trsp}\bU = \0,\! \bV_p \in \mmR^{n\times r}, \!\bV_p^{\trsp}\bV = \0\}.
\end{eqnarray}
\end{small}
\vspace{-0.20in}

Given $\bX \in \mM_{r}$ and ${\bA},{\bB} \in T_{\bX}\mM_{r}$, by defining a metric $g_{\bX}({\bA},{\bB}) = \langle{\bA},{\bB}\rangle$,
$\mM_{r}$ is a \textbf{Riemannian manifold} by restricting
$\langle\bA,\bB\rangle$ to the \emph{tangent bundle}~\cite{AMS2008}.\footnote{The
\emph{tangent bundle} is defined as the disjoint union of all
tangent spaces
$T\mM_{r} = \bigcup_{\bX\in \mM_{r}}\{\bX\} \times T_{\bX}\mM_{r}.$} The norm of a tangent vector
$\bzeta_{\bX} \in T_{\bX}\mM_{r}$ evaluated at $\bX$ is defined as
$||\bzeta_{\bX}|| = \sqrt{\langle\bzeta_{\bX},\bzeta_{\bX}\rangle}$.


Once the metric is fixed, the notion of the gradient of an objective function can be introduced. %
For a Riemannian manifold, the \textbf{Riemannian gradient} of a smooth function $f: \mM_r \rightarrow \mmR$ at $\bX\in\mM_{r}$ is defined as %
the unique tangent vector $\grad{f(\bX)}$ in $T_{\bX}\mM_{r}$, such that %
$\langle\grad{f(\bX)}, \bxi\rangle = \mRMD f(\bX)[\bxi], ~~\forall \bxi \in T_{\bX}\mM_{r}$. %
%
As $\mM_r$ is embedded in $\mmR^{m\times n}$, the Riemannian gradient of $f$ is given as the \textbf{orthogonal
projection} of the gradient of $f$
onto the tangent space. Here, the orthogonal projection of any
$\bZ \in \mmR^{m\times n}$ onto the tangent space $T_{\bX}\mM_{r}$ at $\bX =
\bU\diag(\bsigma)\bV^{\trsp}$ is defined as
%
\begin{eqnarray}\label{eq:projection}
P_{T_{\bX}\mM_{r}}(\bZ): \bZ\mapsto P_{U}\bZ P_{V} + P_{U}^{\perp}\bZ
P_{V} + P_{U}\bZ P_{V}^{\perp}.
\end{eqnarray}
%
where $P_U = \bU \bU^{\trsp}$ and $P_U^{\perp} = \bI - \bU \bU^{\trsp}$.
Moreover, define $P_{T_{\0}\mMLr}(\bZ) = \0$ when $\bX = \0$~\cite{tan2014riemannian}. Letting $\bG = \nabla f(\bX)$ be the gradient of $f(\bX)$
on vector space,
it follows that
%
\begin{eqnarray}\label{eq:grad_Mr}
\grad{f(\bX)} = P_{T_{\bX}\mM_{r}}(\bG).
\end{eqnarray}

The \emph{Retraction} mapping on $\mM_{r}$
relates an an element in the tangent space to a corresponding point on
 the manifold.
 One of the issues associated with such
 retraction mappings is to find the best rank-$r$ approximation to $\bX+\bxi$ in
 terms of
 the Frobenius norm
%
\begin{align}
\label{eq:retrac}
R_{\bX}(\bxi) = & P_{\mM_{r}} (\bX+\bxi) \nonumber \\
= & \argmin_{\bY\in \mM_r} ||\bY - (\bX + \bxi)||_F.
\end{align}

In general, this problem can be addressed by performing SVD on $\bX + \bxi$, which may be computationally expensive.
\begin{remark}
Since $\bxi \!=\! \bU\bM\bV^{\trsp}  \!+\!  \!\bU_p\bV^{\trsp} \!+\!
\bU\bV_p^{\trsp}  \! \in T_{\bX}\mM_{r}$, %
\shijie{$R_{\bX}(\bxi)$ can be efficiently computed} %
as in Algorithm 6 in~\cite{Vandereycken2013}  with efficient QR decompositions on low rank matrices $\bU_p$ and $\bV_p$. %
\shijie{The corresponding time complexity is }%
$14(m + n)r^2 + C_{SVD} r^3$, %
where $r\ll \min(m,n)$~\shijie{and $C_{SVD}$ is a moderate constant}~\cite{Vandereycken2013}.
\end{remark}

%
%
%
%
%
%
%
%

%

\textbf{{Varieties of Low-rank Matrices}}. Note that the submanifold $\mM_{r}$ is open, and the manifold properties break down at the boundary where $\rank(\bX)<r$, and the convergence analysis on $\mM_{r}$ will be difficult accordingly~\cite{schneider2014convergence}. Therefore, it would be more convenient to consider the closure of $\mM_r$:
%
\begin{eqnarray}
\mM_{\leq r} = \{\bX\in \mmR^{m\times n}: \rank(\bX) \leq r\},
\end{eqnarray}
%
which is a real-algebraic variety~\cite{schneider2014convergence}. Let $\mathrm{ran}(\bX)$ be the column space of $\bX$. In the singular points where $\rank(\bX) = s<r$, we will construct search directions in %
\shijie{the tangent cone \cite{schneider2014convergence} (instead of the tangent space)}%
%
\begin{eqnarray}
T_{\bX}\mMLr = T_{\bX}\mM_{s}\oplus\{\bXi_{r-s} \in  \mU^{\perp} \otimes \mV^{\perp} \},
\end{eqnarray}
%
where $\mU = \mathrm{ran}(\bX)$ and $\mV = \mathrm{ran}(\bX^{\top})$.  %
\shijie{Essentially, $\bXi_{r-s}$ is a best $\rank$-$(r\!-\!s)$ approximation of $\bG - P_{T_{\bX}\mM_{s}}(\bG)$, which can be cheaply computed with truncated SVD of rank $({r-s})$.} %
Let $\grad{f(\bX)} \in T_{\bX}\mMLr$ be the projection of $\bG$ on $T_{\bX}\mMLr$. It can be computed by
%
\begin{eqnarray}\label{eq:grad}
\grad{f(\bX)} = P_{T_{\bX}\mM_{s}}(\bG) + \bXi_{r-s}.
\end{eqnarray}

Given a search direction $\bxi\in T_{\bX}\mMLr$, we need perform retraction which finds the best approximation by a matrix of rank at most $r$ as measured in terms of the Frobenius norm, \ieShijie~%
\begin{eqnarray}\label{eq:retrac}
R_{\bX}(\bxi) = \arg\min_{\bY\in \mMLr} ||\bY - (\bX + \bxi)||_F.
\end{eqnarray}
%
Since $\bXi_{r-s} \in \mU^{\perp} \otimes \mV^{\perp}$,  $R_{\bX}(\bxi)$ w.r.t. $\mMLr$ can be efficiently computed with the same complexity as on $\mM_r$%
\shijie{~(see details in the supplementary file).}


\section{Proximal Riemannian Gradient on $\mMLr$}\label{sec:MLr}

Directly solving the general trace-norm regularized problem in (\ref{eq:general_form}) can be computationally expensive due to the unknown rank of variables (regarding fixed-rank methods) or large-rank singular value decompositions (SVDs) (regarding proximal gradient based methods).
\shijie{To make the problem simpler,} %
\shijie{let us first consider the following problem with explicit rank constraint $\rank(\bX)\leq r$}
%
%
\begin{eqnarray}\label{eq:general_form_fixed_rank}
&\min_{\bX,\bE}& ||\bX||_* + \lambda\Upsilon(\bE),  \\ &\mathrm{s.t.} &~ \mA(\bX) + \mB(\bE) = \bD, ~\bX\in \mMLr. \nonumber
\end{eqnarray}
%
Here, the parameter $r$ is supposed to be known. Nevertheless, based on (\ref{eq:general_form_fixed_rank}), we will propose a subspace pursuit paradigm to solve %
\shijienew{the general trace-norm regularized problem in (\ref{eq:general_form});} %
see details in Section 4.

The penalty method is adopted to  deal with the equality constraint $\mA(\bX) + \mB(\bE) = \bD$ in (\ref{eq:general_form_fixed_rank}), and it minimizes %
a penalized function %
over $\mMLr$ in the following form\footnote{If $\bD$ is a vector, the $F$-norm will be replaced by the $\ell_2$-norm.}:

\vspace{-0.20in}
\begin{small}
\begin{eqnarray}\label{eq:form_penaly}
\Psi({\bX,\bE}) = ||\bX||_* + \lambda\Upsilon(\bE) +  \frac{\gamma}{2} ||\mA(\bX) + \mB(\bE) - \bD||_F^2,
\end{eqnarray}
\end{small}
\vspace{-0.20in}

where $\gamma$ is a penalty parameter.
Note that when there are no outliers, we can let $\mB(\bE) = \0$ and $\Upsilon(\bE)=0$, and the objective function $\Psi({\bX,\bE})$ is reduced to
%
\begin{eqnarray}\label{eq:form_penaly_nooutlier}
\Psi(\bX) = ||\bX||_* + \frac{\gamma}{2} ||\mA(\bX) - \bD||_F^2, ~~\bX \in \mMLr.
\end{eqnarray}
%
$\Psi(\bX)$ is also the objective function of the \emph{matrix lasso} problem~\cite{Toh2010APG,ji2009accelerated}, so one can adapt classical proximal methods~\cite{Toh2010APG,Zhouchen2010ALM} to address it. However, %
proximal gradient methods which directly operate on vector spaces could be very expensive if large-rank SVDs are required.

In this section, we extend classical proximal methods on vector space~\cite{Nesterov2007,Toh2010APG,Xiao2013PGH_J}, and propose a \emph{proximal Riemannian gradient} scheme to minimize (\ref{eq:form_penaly}) and (\ref{eq:form_penaly_nooutlier}) by exploiting geometries over the matrix variety $\mMLr$.\footnote{Recall that $\mMLr$ is a closure of the Riemannian submanifold $\mM_r$. Here, we abuse ``Riemannian" for simplicity.}

\subsection{PRG on $\mMLr$ for Non-outlier Cases}
The objective function $\Psi(\bX)$ regarding non-outlier cases is much simpler than $\Psi({\bX,\bE})$. \shijie{When there} %
are no outliers, we solve the following optimization problem:
%
\begin{eqnarray}\label{eq:general_trace_rank}
\min_{\bX} ~~||\bX||_* +  f(\bX), ~~\mathrm{s.t.}~\rank(\bX) \leq r,
\end{eqnarray}
%
where $f(\bX)$ is any smoothing function, for example $f(\bX)=  \frac{\gamma}{2} ||\mA(\bX) - \bD||_F^2$.


To introduce proximal methods on $\mMLr$, similarly as in~\cite{Toh2010APG,Zhouchen2010ALM}, we  introduce a local model of $\Psi(\bX)$ on $\mMLr$ around $\bY\in \mMLr$ but keeping $||\bX||_*$ intact:
%
\begin{eqnarray}\label{eq:local_model}
m_L(\bY; \bX) := ||\bX||_* + f(\bY) + \langle\grad{f(\bY)}, \bxi\rangle + \frac{L}{2}\langle\bxi,\bxi\rangle, \nonumber
\end{eqnarray}
%
where $\bxi \in T_{\bY}\mMLr$ and $\bX = \bY+\bxi$. Note the above local model is different from that on vector spaces (see~\cite{Nesterov2007,Toh2010APG}) in the sense that $\bX - \bY =\bxi$ is restricted on $T_{\bY}\mMLr$. %

Similar to classical proximal gradient methods~\cite{Nesterov2007,Toh2010APG},
our proximal Rimannian gradient method %
\shijie{solves} %
problem (\ref{eq:general_trace_rank}) by minimizing $m_L(\bY; \bX)$ on $\mMLr$ iteratively. In other words, given $\bY = \bX_k$ in the $k$th iteration,  we need to solve the following optimization problem to obtain $\bX_{k+1}$:
%
\begin{eqnarray}\label{eq:local_model}
\min_{\bX}~~m_L(\bY; \bX), ~~\mathrm{s.t.}~~\bX \in \mMLr.
\end{eqnarray}
%
%
For convenience, %
let %
\begin{eqnarray}\label{eq:local_closed_form}
T_L(\bY) := \arg\min_{\bX \in \mMLr} m_L(\bY;\bX)
\end{eqnarray}
be a minimizer of (\ref{eq:local_model}). Then it can be computed as follows.
\begin{lemma}\label{lemma:closed_solution}
Let $\bxi \!\!=\!\!-\grad{f(\bY)}/{L}$. Denoting the SVD of $R_{\bY}(\bxi)$ as $R_{\bY}(\bxi) \!\!=\!\! \bU_+ \diag({\bsigma_+})\bV_+^{\top}$, it follows that $T_L(\bY)  = \bU_+ \diag(\max({\bsigma_+}-1/L,0))\bV_+^{\top}$.
\end{lemma}
Please find the proof in supplementary file.

\begin{remark}
$T_L(\bY)$ can be efficiently computed in the sense that $R_{\bY}(\bxi) = \bU_+ \diag({\bsigma_+})\bV_+^{\top}$ can be cheaply computed without expensive SVDs.
\end{remark}

\begin{algorithm}[H]
\begin{algorithmic}[1]
\REQUIRE $\bX_0$, penalty parameter $\gamma$, parameter $r$, stopping tolerance $\epsilon$.
\STATE For $k = 1, ..., K$\\

\STATE~~Compute $\text{grad} f(\bX_{k-1})$ according to
(\ref{eq:grad_Mr}) or (\ref{eq:grad}).\\


\STATE~~Choose $L_k$ to satisfy (\ref{eq:armijo}), and set %
$\bX_{k} =T_{L_k}(\bX_{k-1})$.
\STATE~~Terminate if stopping conditions are achieved.
\STATE End
\STATE Return $\bX_k$.
\end{algorithmic}
\caption{Proximal Riemannian Gradient for
Solving Problem (\ref{eq:general_trace_rank}).} \label{Alg:PGR}
\end{algorithm}

\shijie{PRG iteratively minimizes a local model of $\Psi$ on $\mMLr$, as shown in Algorithm~\ref{Alg:PGR}. PRG~consists of two major steps} %
1) compute a search direction in Step 3, and 2) update %
\shijie{$\bX_{k}$ according to  $\bX_{k} =T_{L_k}(\bX_{k-1})$.} 
Here, $1/L_k$ can be deemed as the step size, and it can be determined using Armijo line search. %
Specifically, given a descent direction $\bzeta_k \in T_{\bX_k}\mMLr$, $L_k$ is determined such that
%
\begin{eqnarray} \label{eq:armijo}
\Psi(T_{L_k}(\bX_k)) \!\! \leq\!\!  \Psi(\bX_k) + \beta\langle\text{grad} f(\bX_k), \bzeta_k\rangle/L_k,
\end{eqnarray}
%
where $\beta \in (0, 1)$.

\textbf{Optimality condition} of (\ref{eq:general_trace_rank}). A point $\bX^* \in \mMLr$ is a local minimizer of (\ref{eq:general_trace_rank}) if and only if there exists $\bvarsigma \in \partial ||\bX||_*$ such that~\cite{Mishra2013trace}
\begin{eqnarray}
 \grad f(\bX) + \bvarsigma = \0.
\end{eqnarray}

The following lemma guarantees the existence of $L_k$.
\begin{lemma}\label{lemma:step_size}
Let $\bX_k\in \mMLr$, and $\bzeta_k\in T_{\bX}\mMLr$ be a descent direction. %
Then there exists an $L_k$ that satisfies the condition in (\ref{eq:armijo}).
\end{lemma}
\begin{proof}

 Since $\bzeta_k$ is a descent direction, it follows that $\0 \notin \grad f(\bX_k) + \partial ||\bX||_*$ and $\langle\text{grad} f(\bX_k), \bzeta_k\rangle <0$. Since $T_{L}(\bX_k)$ is continuous in $L$, there must exist an  $\widehat{L}$ such that $\Psi(T_{L}(\bX_k)) \!\! \leq\!\!  \Psi(\bX_k) + \beta\langle\text{grad} f(\bX_k), \bzeta_k\rangle/L,$ $\forall L\in [\widehat{L}, +\infty)$.
 \end{proof}

In general,  optimization methods on Riemannian manifolds are guaranteed to be locally convergent, and it is nontrivial to check whether a limit point $\bX^*$ is a global solution or not. However, for PRG,  the limit point $\bX^*$ will be a global solution if $r > \rank(\bX^*)$.
\begin{thm}\label{sec:global}
Let $\{\bX_k\}$ be an infinite sequence of iterates generated by Algorithm \ref{Alg:PGR}. Then every accumulation point of $\{\bX_k\}$ is a critical point of $f$ over $\mMLr$. Furthermore, $\lim_{k\rightarrow \infty}||\grad f(\bX_k) + \bzeta||_F = 0$.  Let $\bX^*$ denote the limit point. In particular, if $\rank(\bX^*)<r$, then we have $\nabla f(\bX^*) + \bzeta = \0$, \emph{i.e.}, $\bX^*$ is a global optimum to (\ref{eq:general_trace_rank}).
\end{thm}
\begin{proof}
Note that $\Psi(\bX)$ is bounded below. The proof can be completed by adapting the proof of Theorem 3.9 in \cite{schneider2014convergence}.
\end{proof}


\textbf{Stopping conditions} of PRG. For simplicity, we stop PRG if the following condition is achieved:
%
\begin{eqnarray}\label{eq:stop_con_inner}
\frac{  {\Psi(\bX_{k-1}  )} - \Psi(\bX_k)  }{ \Psi(\bX_{k-1}) }  \leq
\epsilon,
\end{eqnarray}
%
where $\epsilon$ denotes a tolerance value.

%
%
%


\subsection{Robust PRG on $\mMLr$ with Outliers}
Now, we extend PRG to minimize $\Psi({\bX,\bE})$ in(\ref{eq:form_penaly}) regarding the outlier cases. For convenience, define

\vspace{-0.20in}
\begin{small}
\begin{eqnarray}
f({\bX,\bE}) = \frac{\gamma}{2} ||\mA(\bX) + \mB(\bE) - \bD||_F^2. \nonumber
\end{eqnarray}
\end{small}
\vspace{-0.25in}

We then need to solve the following problem:
%
%
\begin{eqnarray}\label{eq:outlier_case}
\min_{\bX \in \mMLr, \bE}~||\bX||_* + \lambda\Upsilon(\bE) + f({\bX,\bE}).
\end{eqnarray}


Following~\cite{lin2011linearized}, we optimize the two variables $\bX$ and $\bE$
using an alternating approach.
Let the pair $(\bX_k, \bE_k)$ denote the variables obtained from the $k$-iteration. %
At the $(k+1)$th iteration, we update $\bX$ and $\bE$ as below: %

To update $\bX$, we fix $\bE = \bE_k$ and minimize a local model of $\Psi({\bX,\bE})$ w.r.t. $\bX$: %
\begin{align}
m_L(\bX; \bX_k, \bE_k) &:= ||\bX||_* + f(\bX_k, \bE_k) \nonumber \\
                       &  + \langle\grad{f(\bX_k, \bE_k)}, \bxi\rangle + {L}/{2}\langle\bxi,\bxi\rangle, \nonumber
\end{align}
where $\bX = \bX_k+\bxi$, $\bxi \in T_{\bX_k}\mMLr$, and  ${L}$ is a positive number. Let $T_L(\bX_k, \bE_k)$ denote the minimizer of $m_L(\bX; \bX_k, \bE_k)$.  Then $T_L(\bX_k, \bE_k)$ can be computed according to Lemma \ref{lemma:closed_solution}, where $L$ is determined by Armijo line search to make a sufficient decrease of the objective.

To update $\bE$, we fix $\bX = \bX_{k+1}$ and solve a problem: %
%
\begin{eqnarray}\label{eq:update_E}
\min_{\bE} ~~\lambda\Upsilon(\bE)
+ \frac{\gamma}{2}||\mA(\bX_{k+1}) + \mB(\bE)-\bD||_F^2.
\end{eqnarray}

Solving this problem with general $\mB$ would be very difficult. However, for  MR and LRR, $\mB(\bE)=\bE$ and $\Upsilon(\bE)$ is either $\|\bE\|_{1}$ or $\|\bE\|_{2,1}$. As a result, the problem (\ref{eq:update_E}) has a closed-form solution. %
Let us define $\bB_k =\bD-\mA(\bX_{k+1})$. Then $\bB_k$ is a vector for MR and a matrix in the form of $\bB_k = [\mathbf{b}^k_1,\ldots, \mathbf{b}^k_{n}]$ for LRR. %
The closed-form solution, denoted by $S_{\lambda}(\bB_k)$, is shown in Table \ref{table:update_E}. %
In cases where the problem (\ref{eq:update_E}) cannot be solved in closed-form, one may adopt iterative procedures to solve it.

The detailed algorithm, which is referred to as robust PRG (RPRG),  is shown in Algorithm \ref{Alg:PRG_penlty}. Due to the possible ill-conditioned issues,\footnote{When $\lambda$ is very small or $\gamma$ is very large, $\|\bE\|_{1}$ can be very large at the beginning due to the thresholding in Table 1, making $\bX_1$ far from its optimum.} we apply a homotopy continuation technique to accelerate the convergence speed. Starting from an initial guess $\lambda_0$, we set $\lambda_k = \min(\lambda_0 \rho^{k-1},{\lambda})$ and compute $\bE_{k} =S_{\lambda_k}(\bX_{k}, \bE_{k-1})$, %
where $\rho$ is chosen from $(0, 1)$. %
Clearly, $\lambda_k$ is %
non-increasing
w.r.t. $k$.

 %



\begin{table}[t!]
\vspace{-0.05in}
\center \caption{Computation of $S_{\lambda}(\bB)$.}\label{table:update_E}
\begin{scriptsize}
\begin{tabular}{c|cc|c}
\hline
$\Upsilon(\bE)$  &      \multicolumn{ 2}{|c|}{MR: $||\bE||_1$} &      {LRR: $||\bE||_{\twoone}$}       \\
\cline{1-4}
$\!\!\!\!S_{\lambda}(\bB)\!\!\!\!$\!\!\!\!& \multicolumn{2}{|c|}{$\!\!\!\sgn(\bB)\odot\max(|\bB|-\frac{\lambda}{\gamma}, \0)\!\!\!\!$} & \!\!\!\!$[S_{\lambda}(\bB)]_i =\!\!\frac{\max(\|\mathbf{b}_i\|-\frac{\lambda}{\gamma},0)}{\|\mathbf{b}_i\|}\mathbf{b}_i, \forall i$    \\
\hline
\end{tabular}
\end{scriptsize}
\vspace{-0.150in}
\end{table}

\begin{algorithm}[H]
\begin{algorithmic}[1]
\REQUIRE Initial  $(\bX_0, \bE_0)$, parameter $\lambda$ and $\gamma$, initial $\lambda_0$,  parameter $r$, $\rho \in (0,1)$, stopping tolerance $\epsilon$.

\STATE For $k = 1, ..., K$

\STATE  Let $\lambda_k = \max(\lambda_0 \rho^{k-1},\lambda)$.

\STATE Compute $\text{grad} f(\bX_{k-1},\bE_{k-1})$ by
(\ref{eq:grad_Mr}) or (\ref{eq:grad}).


\STATE  Choose $L_k$ by Armijo line search. Set $\bX_{k} =T_{L_k}(\bX_{k-1},\bE_{k-1})$. \\

\STATE  Compute $\bE_{k} \!\!=\!\!S_{\lambda_k}(\bB_{k-1})$, where $\bB_{k-1}\!=\!\!\bD-\mA(\bX_{k})$. \\

\STATE Terminate if stopping conditions
are achieved.
\STATE End
\STATE Return $(\bX_k, \bE_k)$.
\end{algorithmic}
\caption{Robust PRG for Solving Problem (\ref{eq:general_form_fixed_rank}).} \label{Alg:PRG_penlty}
\end{algorithm}

We discuss the convergence as follows. %
\begin{prop}\label{prop:conver_RPRG}
Let $\Psi({\bX_k,\bE_k}) = ||\bX||_* + \lambda_k\Upsilon(\bE) + f({\bX,\bE})$, and $\{(\bX_k, \bE_k)\}$ be an infinite sequence of iterates generated by Algorithm \ref{Alg:PRG_penlty}. It follows that $\Psi(\bX_{k+1},\bE_{k+1}) \leq\Psi(\bX_{k},\bE_{k})$, and $\{(\bX_k, \bE_k)\}$ converges to a limit point $(\bX^*, \bE^*)$.
\end{prop}
Please find the proof in supplementary file. Thee stopping condition in (\ref{eq:stop_con_inner}) can be extended to RPRG by replacing $\Psi(\bX)$ with $\Psi({\bX,\bE})$.
\section{Subspace Pursuit for Solving Problem (\ref{eq:general_form})}\label{sec:framework}
Both PRG and RPRG methods operate on $\mMLr$, which rely on the knowledge of $r$. Unfortunately, the parameter $r$ is usually unknown. Based on Theorem \ref{sec:global}, one should set a sufficiently large $r$ such that $r \geq \rank(\bX^*)$, where $\bX^*$ is an optimal solution of problem (\ref{eq:general_form}). However, the computational cost will dramatically increase if $r$ is too large.

Regarding the above issues, we propose a subspace pursuit (SP) paradigm to address problem (\ref{eq:general_form}). To introduce SP, we bring in an additional integer $\kappa$ which is \textbf{assumed to be several times smaller} than $\rank(\bX^*)$.
Taking the outlier cases for example, instead of doing RPRG with a large $r$, we gradually increase the rank of $\bX$ from $\rank(\bX) = 0$ (\egShijie~ $\bX=\0$) by a small value $\kappa$, and perform RPRG to the following subproblem with increased $t$.
%
%
\begin{align}\label{eq:form_penaly_fixed_rank}
\min_{\bX,\bE}~~\Psi({\bX,\bE}), ~~\mathrm{s.t.} ~~\rank(\bX) \leq t\kappa,
\end{align}
%
where $t$ denotes the iteration index. Essentially, the subspace pursuit addresses problem (\ref{eq:general_form}) by solving a series of subproblems in (\ref{eq:form_penaly_fixed_rank}) using RPRG on $\mM_{\leq t\kappa}$ (see Step 8 of Algorithm \ref{Alg:Act_PRG_penlty}), where $t=1, ..., T$ and $T$ denotes the maximum number of iterations.




%

%
%
%
%
%
%


\begin{algorithm}[h!]
\begin{algorithmic}[1]
\REQUIRE  Parameters $\kappa$, $\lambda$ and $\gamma$, initial $\lambda_0$, $\chi$, $\rho$ (where $\chi< \rho$), and stopping tolerance $\epsilon$ and $\varepsilon$.

\STATE Initialize  $\bX^0=\0$, $\bE^0 = \0$,  and $\lambda_0^{0} = \lambda_0$.

\STATE For $t = 1, ..., T$
\STATE Let $r = t\kappa$ and $s = (t-1)\kappa$. \STATE Let $\lambda_0^t = \lambda^{t-1}$, and $\lambda^t = \max(\lambda_0 \chi^{t},\lambda)$.

\STATE Compute $\bG = \gamma\mA^*(\mA(\bX^{t-1}) + \mB(\bE^{t-1})-\bD)$.
\STATE Compute $\grad{f(\bX^{t-1},\bE^{t-1})} \!\!=\!\! P_{T_{\bX^{t-1}}\mM_{s}}(\bG) + \bXi^{t-1}_{\kappa}$.\\

\STATE Choose $L_t$ to satisfy (\ref{eq:armijo}).
Set $\bX_0^{t} =T_{L_t}(\bX^{t-1}, \!\bE^{t-1})$ and $\bE_0^{t} =S_{\lambda_0^t}(\bB^{t-1})$, where $\bB^{t-1}\!=\!\!\bD-\mA(\bX^{t})$. \\

\STATE Update $(\bX^t, \bE^t)$ by calling RPRG to address  (\ref{eq:form_penaly_fixed_rank}) with initial input $(\bX_0^{t}, \bE_0^{t})$, $r$, $\lambda_0^t$, $\lambda^t$, and $\epsilon$. %

\STATE Terminate if stopping conditions
are achieved.
\STATE End

\end{algorithmic}
\caption{Subspace Pursuit for Solving (\ref{eq:general_form}).}\label{Alg:Act_PRG_penlty}
\end{algorithm}

At the $t$th iteration, SP either stops or increase $r$ from $t\kappa$ to $(t+1)\kappa$ and the domain of $\bX$ changes from a Riemannian manifold $\mM_{t\kappa}$ to $\mM_{\leq (t+1)\kappa}$.  As a result, we need to compute $\bXi_{r-s}$ in Step 6 to calculate $\grad{f(\bX^{t-1},\bE^{t-1})}$ according to equation (\ref{eq:grad}).

To accelerate the convergence speed, two techniques are critical, i.e. the warm start  and homotopy continuation techniques. To warm start, we prepare a good initial guess of $(\bX_0^{t}, \bE_0^{t})$ for RPRG in Steps 4-6 (Steps 4-6 are exactly Steps 4-5 in~Algorithm \ref{Alg:PRG_penlty}). To apply the continuation technique in RPRG. To facilitate its usage in SP, in Step 8, we adaptively set the initial $\lambda_0$ and target $\lambda$ for RPRG by $\lambda_0^t$ and $\lambda^t$ (see Step 4), respectively. Note the parameter $\chi$ should be smaller than $\rho$ in RPRG.  






Note that PRG can be also incorporated into the SP framework. For convenience, we refer SP with RPRG and PRG to as SP-RPRG and SP-PRG, respectively.

SP increase the rank by $\kappa$ iteratively. Due to  limited size of $\bX$, SP will be stopped in limited steps. Moreover, the objective value $\Psi(\bX,\bE)$ monotonically decrease  w.r.t. $t$.
\begin{prop}\label{prop:decrease}
 Let $\{\bX^t, \bE^t\}$ be the sequence generated by Algorithm \ref{Alg:Act_PRG_penlty}. Then we have
 \begin{eqnarray}
 \Psi(\bX^{t},\bE^{t}) \leq \Psi(\bX^{t-1},\bE^{t-1}) - \beta||\bXi^{t-1}_{\kappa}||_F^2/L_t.
\end{eqnarray}
\end{prop}%
Please find the proof in supplementary file.

\textbf{Stopping conditions} of SP. Let $\bX_t^*$ be the limit point of RPRG at the $t$th iteration.  According to Theorem \ref{sec:global}, if $\rank(\bX_t^*)< t\kappa$, $\bX_t^*$ must be a global solution. As a result, SP will stop at the $t$th iteration. According to Proposition \ref{prop:decrease}, if $||\bXi^{t-1}_{\kappa}||_F^2$ is very small, then there is no need to proceed. Then we may also stop the iterations if the following condition is achieved:
%
\begin{eqnarray}\label{eq:stop_con_outer}
\frac{{\Psi(\bX^{t-1}, \bE^{t-1})} - \Psi(\bX^t,\bE^{t})}{\kappa\Psi(\bX^{t-1},\bE^{t-1})}\leq
\varepsilon,
\end{eqnarray}
%
where $\varepsilon$ denotes a tolerance value.

\subsection{Parameter settings}\label{sec:parameter_set}

For convenience of parameter setting, we suggest choosing the penalty parameter $\gamma$ in (\ref{eq:form_penaly_fixed_rank}) according to $\gamma = 1/({\nu}\sigma_{1})$, where $\nu$ is a scaling factor, and $\bsigma$ denotes the singular vector of $\mA^*(\bD)$.\footnote{The setting of $\gamma$ is consistent with the setting of $\mu$ in \emph{matrix lasso} in~\cite{Toh2010APG}, where $\gamma = 1$, and $\mu = \nu\sigma_{1}$ is suggested in general.} For robust cases, the parameter $\lambda$ in (\ref{eq:form_penaly}) is chosen by $\lambda = \delta \gamma d_m$, where $d_m$ denotes the mean of $|\bD|$. The integer $\kappa$ is chosen such that $\sigma_{i} \geq \eta \sigma_1, ~\forall i \leq \kappa,$ and $\sigma_{\kappa+1} < \eta \sigma_1$. Without loss of generality, we suggest setting $\nu \in (0.0001, 0.01)$, $\delta \in (0.01, 1]$, and $\eta \in (0.5, 0.9)$. One may also apply cross-validations to choose $\nu$ and $\delta$ regarding model parameters.  Lastly, in SP, we do not need to optimize each subproblem accurately. Therefore, in SP we suggest setting $\epsilon=0.01$ for PRG and RPRG.


%

\subsection{Complexity Analysis}

The complexity of SP mainly includes two parts, \ieShijie~the computation of $\bXi^{t}_{\kappa}$ which can be done by truncated SVD of rank $\kappa$ and the subproblem optimization by PRG or RPRG.

Here, we focus on the complexity of SP on MR. At the $t$th iteration of SP, the complexity of PRG or RPRG is $O((m +n)(t\kappa)^2+l t\kappa)$, where $t\kappa\leq r + \kappa$. For sufficiently sparse matrices like in MR, the truncated SVD of rank $\kappa$ in SP can be completed in $O((m+n)\kappa)$ using PROPACK~\cite{Larsen:2004}; while the truncated SVD in existing proximal gradient based methods takes $O((m+n)r)$, where $\kappa$ is several times smaller than $r$. Note that in general SP  needs only $\lceil r/\kappa\rceil$ times of truncated SVDs. Therefore, SP is cheaper than existing proximal gradient based methods on MR. The complexity comparison on LRR can be found in supplementary file.

\section{Related Work}

The authors in \cite{Mishra2013trace} exploited Riemannian structures and presented a trust-region algorithm to address trace-norm minimizations.
The proposed method, denoted by MMBS, alternates between fixed-rank optimization and rank-one updates. However, this method shows slower speed than APG on large-scale problems \cite{Mishra2013trace}. The authors in \cite{Liu2014} proposed a Grassmannian manifold method to address trace-norm minimizations on a fixed-rank manifold. In general, this method has similar complexity to {ScGrassMC} that also operates on Grassmannian manifold~\cite{Ngo2012}

Active subspace methods or greedy methods, that increase the rank by one per iteration, have gained great attention in recent years~\cite{Hazan2008,Jaggi2010,Shamir2011,Laue2012}. However, these methods usually involve expensive subproblems, and might be very expensive when the
true
rank is high. For example, Laue's method~\cite{Laue2012} needs to solve nonlinear master problems using %
\shijie{the BFGS method}, %
which can be very expensive for large-scale problems. More recently, \citet{hsieh2014nuclear} proposed a novel active subspace selection method for solving trace-norm regularized problems. %
However, this method may %
\shijie{suffer from} slow convergence speed %
due to the approximated SVDs and inefficient solvers for the subproblem optimization.  In \citet{tan2014riemannian}, the authors proposed a Riemanian pursuit (RP) algorithm which increases the rank more than one. However, this algorithm cannot deal with  trace-norm regularized problems.

\section{Experimental Results}
We %
\shijie{evaluate} %
the proposed methods %
on %
\shijie{for} %
two classical trace-norm based tasks, namely low-rank matrix completion and LRR based clustering.
All the experiments are conducted in Matlab (R2012b) on a PC installed a 64-bit operating system with an Intel(R) Core(TM)
i7 CPU (3.2GHz with single-thread mode) and 64GB memory.

\begin{figure*}[t!]
\center \subfigure[Relative objective difference  w.r.t. time.]{\label{fig:ObjDiffToyNoise}\includegraphics[trim = 1mm 0mm 1mm
0mm, clip, width=0.750\columnwidth]{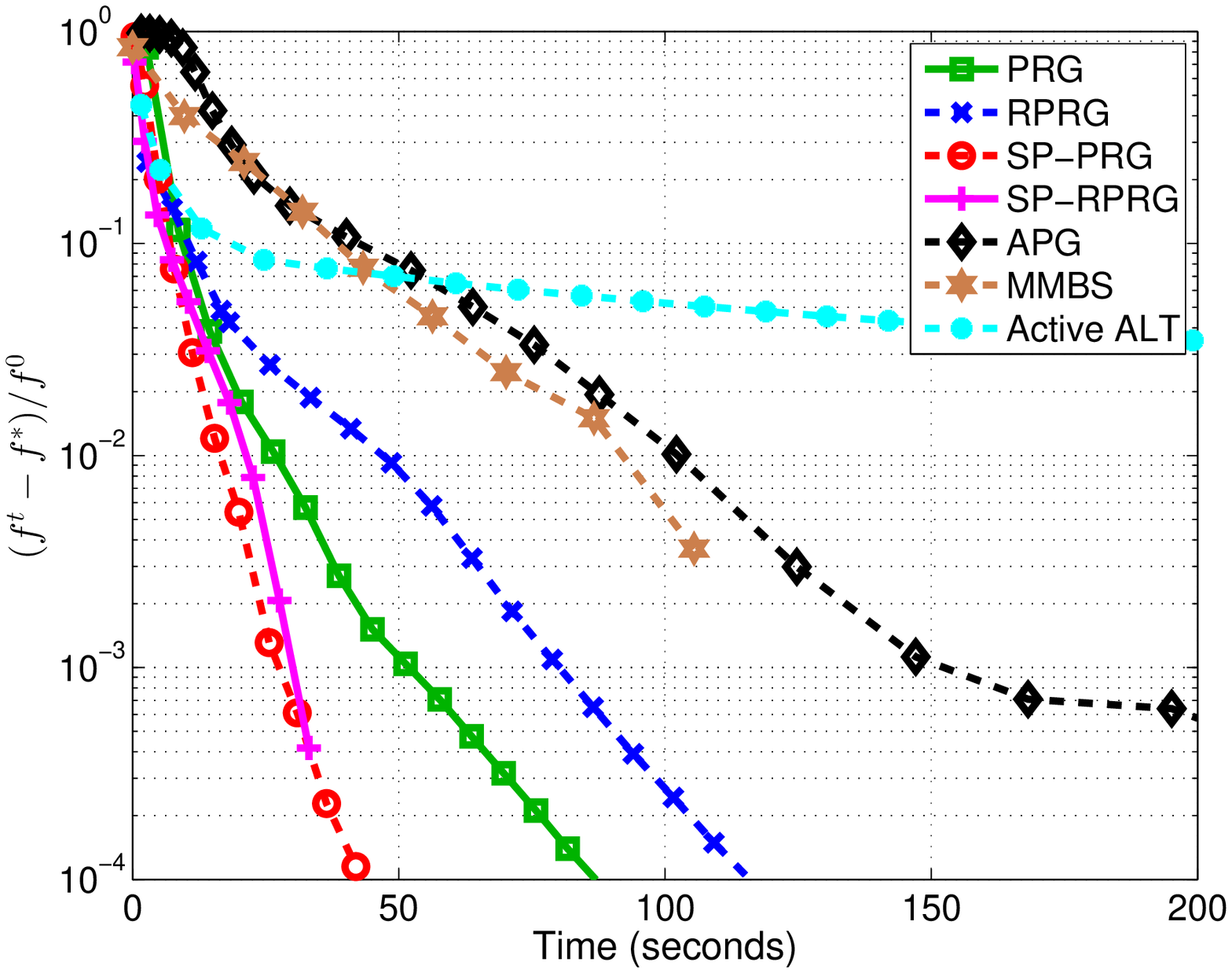}}
\subfigure[Testing RMSE values w.r.t. time]{\label{fig:TestRMSEToyNoise}
\includegraphics[trim = 1mm 0mm 1mm
0mm, clip, width=0.750\columnwidth]{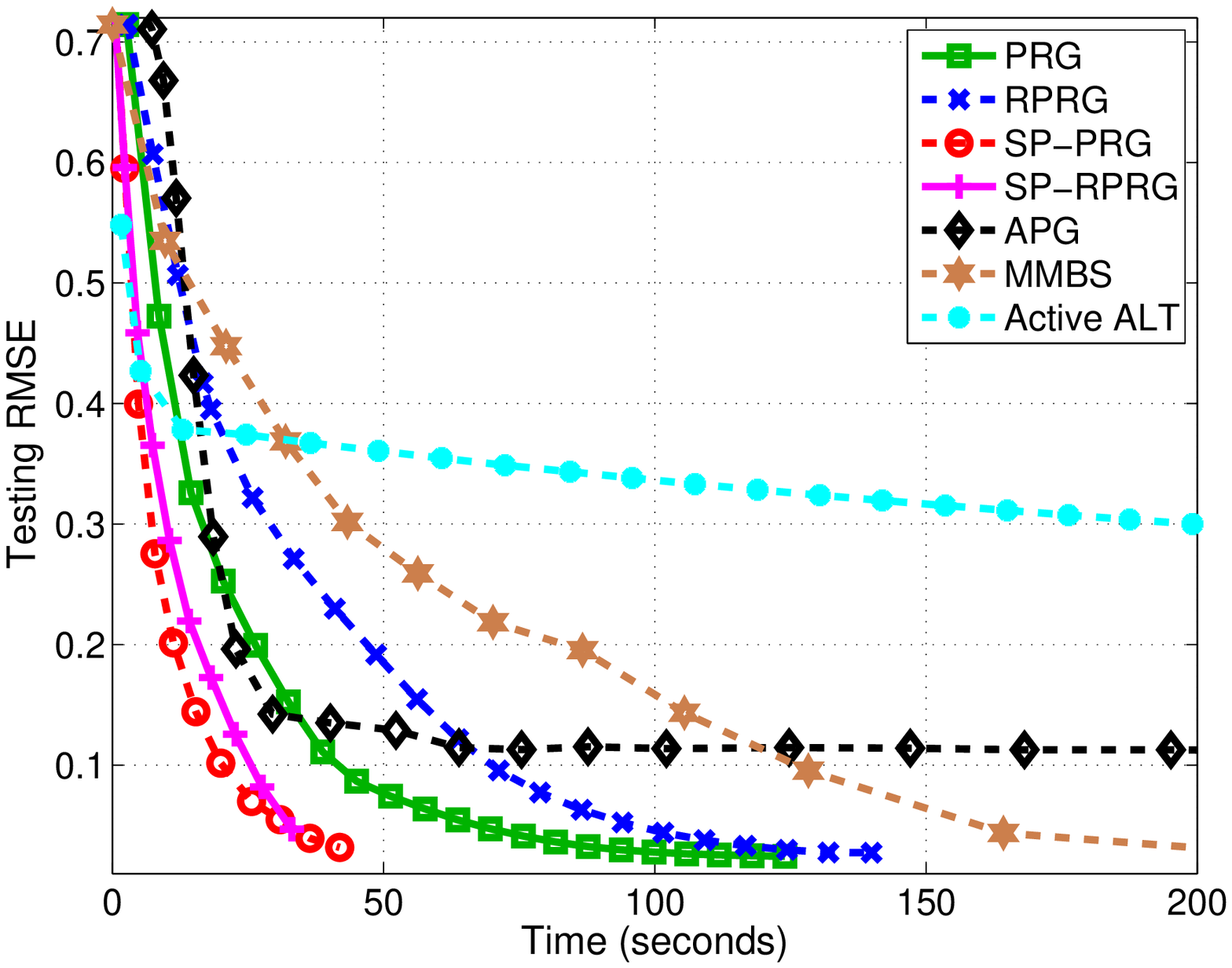}}
\caption{%
\shijienew{Performance of various methods on \textbf{TOY1}.} %
}\label{fig:toy1}
\end{figure*}

\begin{figure*}[t!]
\center \subfigure[Relative objective difference w.r.t. time.]{\label{fig:ObjDiffOutlier}\includegraphics[trim = 1mm 0mm 1mm
0mm, clip, width=0.750\columnwidth]{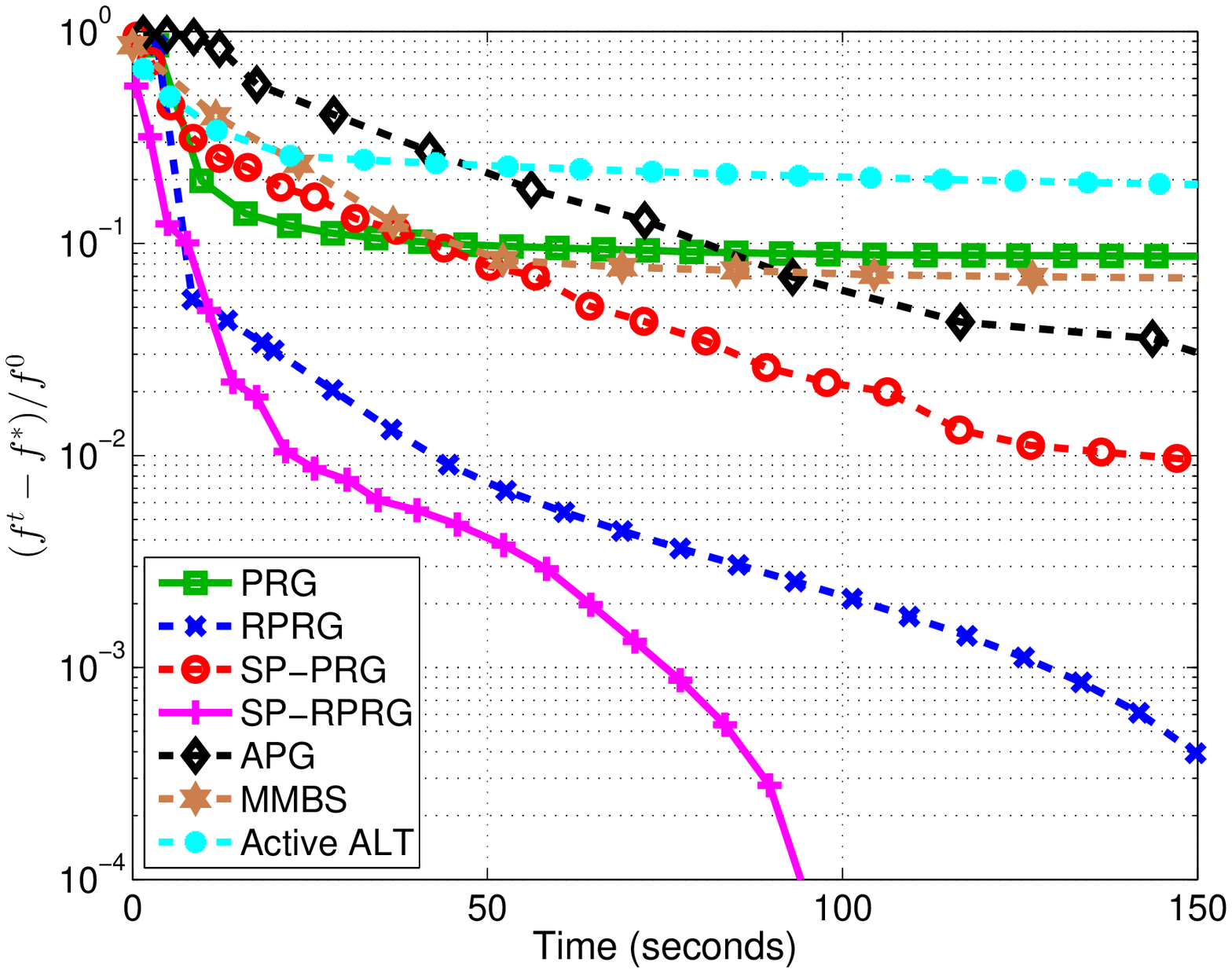}}
\subfigure[Testing RMSE values w.r.t. time]{\label{fig:TestRMSEOutlier}\includegraphics[trim = 1mm 0mm 1mm
0mm, clip, width=0.750\columnwidth]{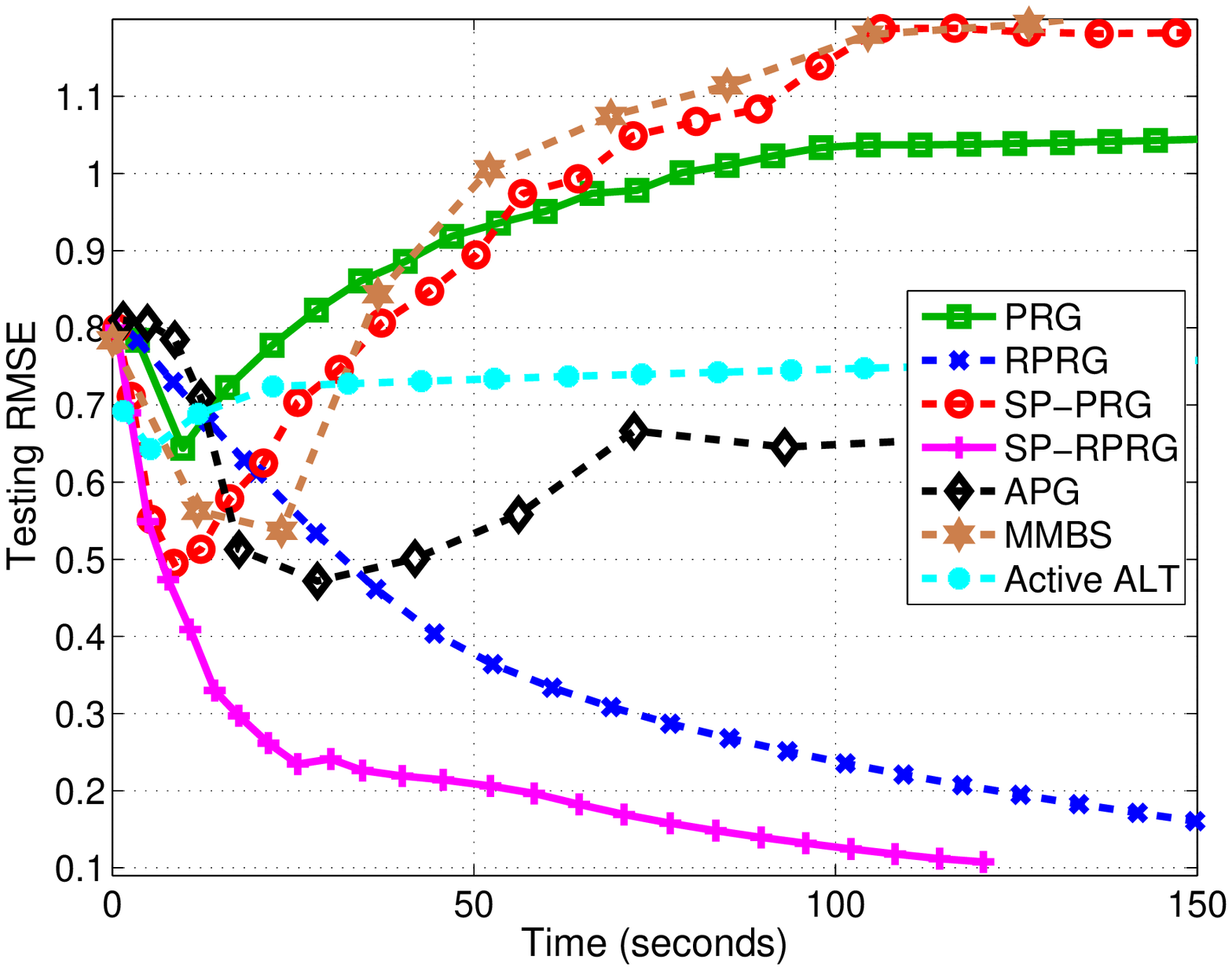}}
\caption{%
\shijienew{Performance of various methods on \textbf{TOY2}, where $5\%$ of data are  disturbed by severe outliers. While other methods over-fit on this data, the proposed robust PRG (RPRG) and SP-RPRG still achieve promising testing RMSE.} %
}\label{fig:toy2}
\end{figure*}

\subsection{Experiments on Matrix Completion}

We study the performance of proposed methods, namely PRG, RPRG, SP-PRG and SP-RPRG, on the matrix completion task.
Three state-of-the-art trace-norm based methods, e.g.
APG~\cite{Toh2010APG}, MMBS~\cite{Mishra2013trace}, and Active ALT~\cite{hsieh2014nuclear}, are adopted as baselines. Moreover, to study the efficiency of proposed methods on matrix completion, we also compare several efficient fixed-rank methods, including RP~\cite{tan2014riemannian}, LMaFit~\cite{Wen2012}, ScGrassMC~\citep{Ngo2012},  LRGeomCG~\cite{Vandereycken2013}, qGeomMC~\cite{Mishra2012}.
\footnote{APG is available from \url{http://www.math.nus.edu.sg/~mattohkc/NNLS.html}; RP is available from \url{http://www.tanmingkui.com/rp.html}; Active ALT is available from \url{http://www.cs.utexas.edu/~cjhsieh/}; MMBS, {LMaFit}, {ScGrassMC}, {qGeomMC} and
{LRGeomCG} are available from:
\url{http://www.montefiore.ulg.ac.be/~mishra/fixedrank/fixedrank.html}.} %
\shijie{We do not report the results of some methods, } %
such as IALM~\cite{Zhouchen2010ALM} and method in \cite{Liu2014}, since they are either slower than the compared methods in this paper or the source codes are not available.


We adopt the root-mean-square error (RMSE) testing set as a major %
\shijie{evaluation metric:} %
$\textrm{RMSE} =
\|\mP_{\Omega}(\bD-\bX^*)\|_F/\sqrt{(|\Omega|)}$, where ${\bX}^*$ denotes the recovered matrix, and $\Omega$
denotes the index set of a testing set,  and $\mP_{\Omega}$ denotes the orthogonal projection onto $\Omega$~\cite{Vandereycken2013}.

\begin{figure*}[t!]
\center \subfigure[Relative objective difference  w.r.t. time.]{\label{fig:RelObjVal}\includegraphics[trim = 1mm 0mm 1mm
0mm, clip, width=0.750\columnwidth]{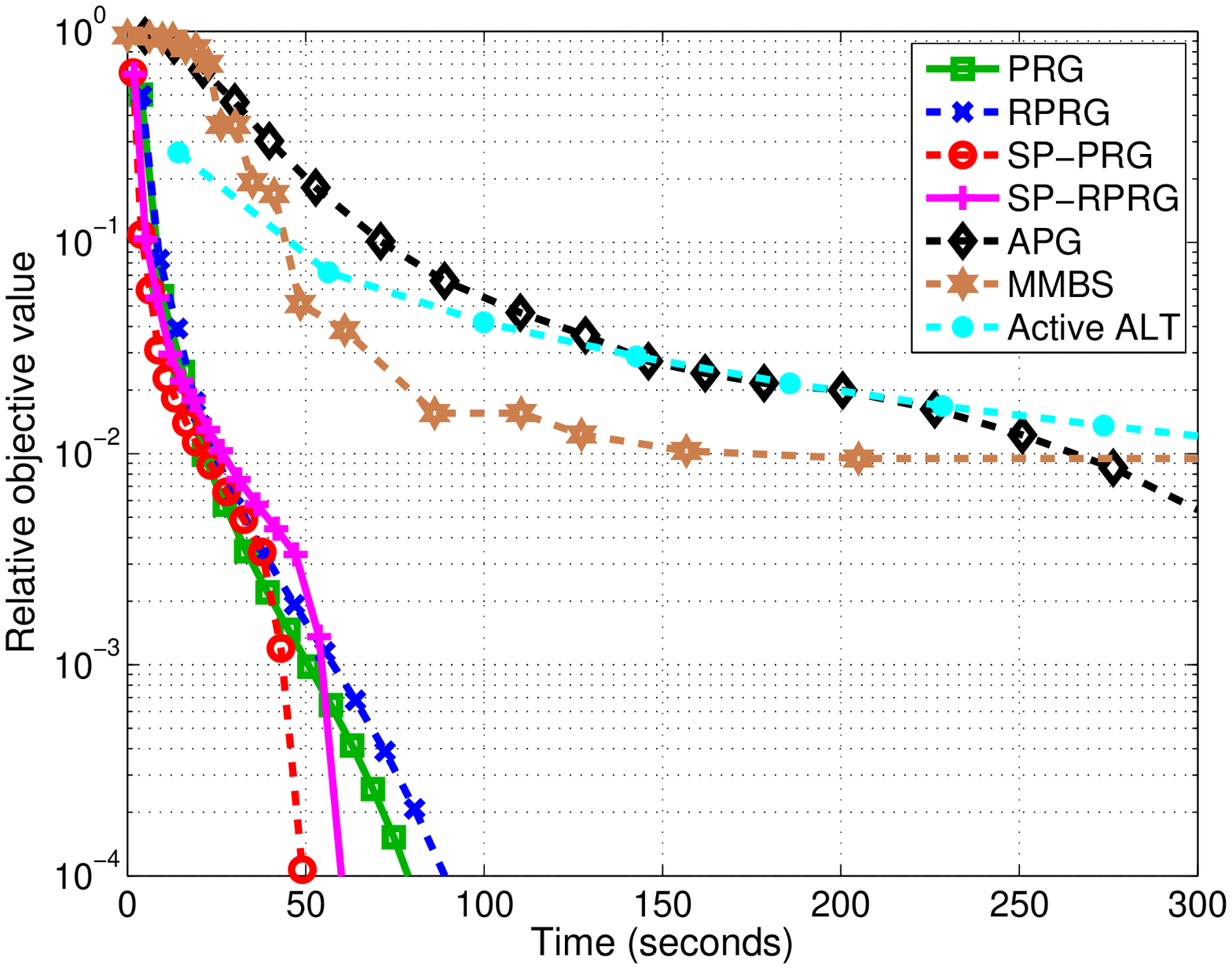}}
\subfigure[Testing RMSE values w.r.t. time]{\label{fig:TestRMSE}\includegraphics[trim = 1mm 0mm 1mm
0mm, clip, width=0.750\columnwidth]{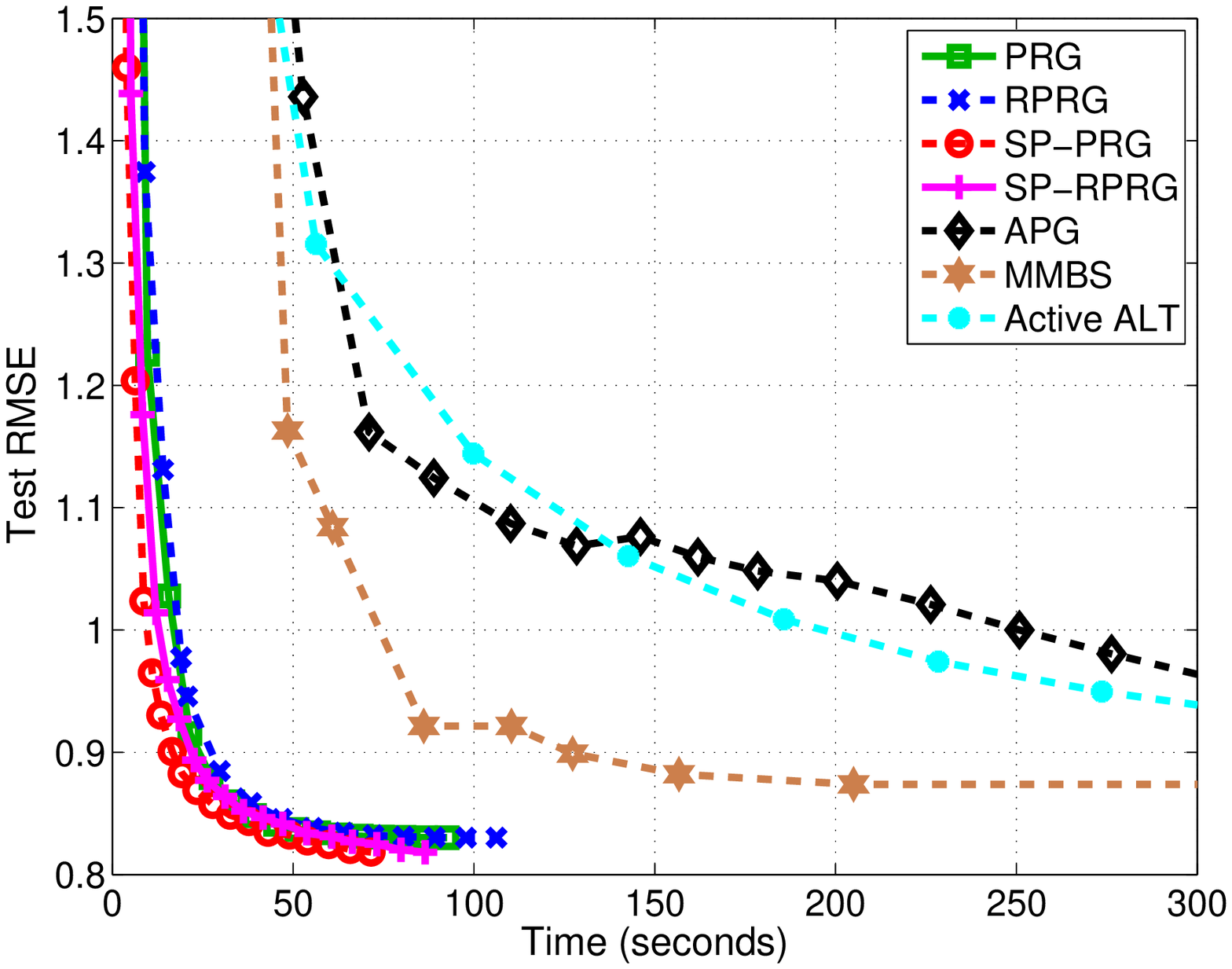}}
\caption{Performance of various methods on Movie-10M data set.}\label{fig:movie}
\end{figure*}
\subsubsection{Synthetic Experiments}\label{sec:synthetic}
Following~\citep{Ngo2012,tan2014riemannian}, we generate
ground-truth low-rank matrices ${\bD} =
{\bU} \diag({\bsigma}) {\bV}^{\top} \in \mmR^{m\times m}$ %
\shijie{of rank $r$}, %
where $\bU \in \textrm{St}_{r}^{m}, \bV \in
\textrm{St}_{r}^{m}$, $m=5000$, $r=50$ and $\bsigma$ is a %
\shijie{$50$-dimensional vector with its entries} %
sampled from %
\shijie{a {uniform} distribution $[0, 1000]$.} %
We sample %
$\shijie{l = \omega r(2m-r)}$ %
entries from ${\bD}$ uniformly as the observations stored in $\bd \in \mmR^l$, where $\omega$ is an oversampling
factor~\cite{Zhouchen2010ALM}. Here, we set $\omega = 2.5$. We study two toy data sets whose observations are perturbed by two kind of noises:
In the first toy data set \textbf{TOY1}, each
entry of $\bd$ is perturbed by
additive Gaussian noise of magnitude
 $0.01
{\|\bd\|_2}/{\|\bn\|_2}$, where $\bn \in \mmR^l$ is a Gaussian vector with each entry being sampled from $N(0, 1)$;
\shijienew{The second toy data \textbf{TOY2} is obtained based on TOY1, by further perturbing $5\%$ of the observations with outliers uniformly sampled from $[-10, 10]$.} %
In %
\shijienew{these synthetic experiments,} %
we set $\nu = 0.005$ $\eta = 0.65$, and $\delta = 0.1$.

Three trace-norm based  methods APG,  MMBS and Active ALT, are adopted as the baselines.
The \textsl{{Relative objective difference}} and \textsl{Testing RMSE} w.r.t. time on \textbf{TOY1} and  \textbf{TOY2} are reported in Figures \ref{fig:toy1} and \ref{fig:toy2}, respectively.

According to Figure \ref{fig:ObjDiffToyNoise}, our proposed  PRG, RPRG, SP-PRG and SP-RPRG converge much faster than the comparators, %
\shijienew{and SP-PRG and SP-RPRG improve upon their counterparts (\ieShijie~PRG and RPRG) significantly.} %
From Figure \ref{fig:TestRMSEToyNoise}, the testing RMSE shows similar trends to the objective values. %
\shijienew{Note that, our methods thus achieve low RMSE values in very short times.} %
In general, the Active ALT method is slower than others, %
which may be%
\shijienew{due to} %
the approximated SVDs and inefficient solvers for the subproblem optimization.

From Figure \ref{fig:ObjDiffOutlier}, on \textbf{TOY2} which is disturbed by outliers,  our proposed method in general converges faster than the baselines. %
However, from Figure \ref{fig:TestRMSEOutlier}, only the proposed RPRG and SP-RPRG achieve promising testing RMSE values. %
\shijienew{While other methods over-fit} %
the data after several iterations due to \shijienew{the outliers.} %
%
Not also that %
SP-RPRG converges faster than its counterpart RPRG. %
\shijienew{This observation demonstrates} the effectiveness and efficiency of \shijienew{our} propose methods.


\subsubsection{Experiments on Real-world Data}\label{sec:synthetic}

We study the performance of {SP-PRG} and SR-PRG on three collaborative filtering data sets: MovieLens
with 10M ratings (denoted by Movie-10M)~\cite{HERLOCKER1999}, %
Netflix Prize dataset~\cite{KDD2007} %
and  and Yahoo! Music Track 1 data set~\cite{kddcup2011}. The statistics of these data sets are recorded in Table \ref{table:stat_data}. %

 \begin{table}[h!]
\center \caption{Statistics of
datasets.}\label{table:stat_data}
 \begin{scriptsize}
\begin{tabular}{c|c|c|c}
\hline
       {Data set} & m & $n$  &    $|\Omega|$            \\
\hline
       {Movie-10M} & 71,567 & 10,677 &     10,000,054         \\

  {Netflix} & 48,089 &  17,770 &     100,480,507              \\

     {Yahoo} & 1,000,990 & 624,961&       252,800,275              \\
\hline
\end{tabular}
\end{scriptsize}
\end{table}

In the first experiment, we only compare with the three trace-norm based methods,  e.g. {APG}~\cite{Toh2010APG}, MMBS~\cite{Mishra2013trace} and Active ALT~\cite{hsieh2014nuclear} on Movie-10M. We report the change of \textsl{{Relative objective difference}} and \textsl{Testing RMSE} w.r.t. time  in Figure \ref{fig:movie}. Here, we randomly choose $80\%$ of the ratings as training set and the remainder as the testing set.  From Figures \ref{fig:RelObjVal} and \ref{fig:TestRMSE}, our proposed methods show much faster convergence speed as well as faster decreasing of testing RMSE values.


\shijie{In the second experiment, the baseline methods include {APG}~\cite{Toh2010APG}, MMBS~\cite{Mishra2013trace}, {LRGeomCG}~\cite{Vandereycken2013}, {qGeomMC}~\cite{Mishra2012}, {Lmafit}~\cite{Wen2012}, Active ALT~\cite{hsieh2014nuclear},  ScGrassMC~\citep{Ngo2012} and RP~\cite{tan2014riemannian}.}~%
The ranks returned by {SP-PRG} are used as the rank estimations for fixed-rank methods, such as {ScGrassMC}, {qGeomMC} and {Lmafit}. We set $\nu = 0.001$ $\eta = 0.65$, and $\delta = 0.7$. %
Following~\cite{Laue2012,Shamir2011,Jaggi2010}, we report
the testing RMSE of different methods over 10 random 80/20
training/testing partitions. %

Comparison results are shown in Table~\ref{table:real_large}. Note that we did not obtain the results of some methods on two larger data sets, i.e. Netflix and Yahoo, due to the very expensive computation cost. We thus leave the results blank. %
According to the table, %
\shijienew{the proposed SP-PRG and SP-RPRG methods achieve better testing RMSE than RP with comparable time. Note that RP relies on carefully designed stopping conditions to induce low-rank solutions, and cannot deal with outliers~\cite{tan2014riemannian}.} %
In particular, SP-PRG and SP-RPRG achieve significant improvements \shijienew{in terms of testing RMSE} on the Yahoo data set.


\begin{table}[!htb]
\center \caption{Experimental results on real-world
datasets, where time is recorded in seconds.}\label{table:real_large}
\begin{scriptsize}
\begin{tabular}{c|c|c|c|c|c|c}
\hline
\shijienew{\multirow{2}{*}{Method}} &
\multicolumn{ 2}{|c|}{Movie-10M} & \multicolumn{ 2}{|c}{Netflix}  & \multicolumn{ 2}{|c}{Yahoo}\\
\cline{2-7}
& RMSE &  Time & RMSE & Time & RMSE & Time\\
\hline
APG     	 	 	&	1.094	&	810.01 	 &  1.038	&	2883.80  &  --	&	 -- 		 \\
LRGeomCG	 	 	&	0.823	&	57.67 	 & 0.860	&	2356.86  &25.228	&18319	  \\
QgeomMC     	 	&	0.836	&	96.41 	 & 0.897	&	9794.75  &  24.167	&	 82419 	 \\
Lmafit	    	 	&	0.838	&	133.86 	 & 0.876	&	2683.73  &24.368	&24349\\
ALT         	 	&	0.855	&	917.17 	 &  --	    &	 -- 	 &  --	&	 -- 	 \\
MMBS       	 		&	0.821	&	441.10 	 &  --	    &	 --	     &  --	&	 -- 	\\
ScGrassMC   	 	&	0.845	&	216.07 	 & 0.892 	&	4522.68  &24.954	&37705	 \\
RP	         	 	&	0.818	&	46.56 	 & 0.858	&	1143.02  &23.451	&12456\\
\hline
SP-PRG    	 	 	&	{0.817}	&	53.42 	 &  {0.855}	&	1057.35  &{22.644}	&15972\\
SP-RPRG 	&	{0.815}	&	67.73 	 &  {0.857}	&    1245.15 	&{22.537}	&17263 \\

\hline
\end{tabular}
\end{scriptsize}
\end{table}

\subsection{Experiments on LRR Based Subspace Clustering}
\shijie{To compare our proposed SP-RPRG method with the existing LRR solvers in \cite{lin2011linearized,LRRPAMI},} %
we conduct experiments on the Extended Yale Face Database B (\textbf{ExtYaleB}) for face clustering, and the Human Activity Recognition Using Smartphones dataset {(\textbf{HARUS})} \cite{anguita2012human} for human activity clustering.~
Following \cite{LRRPAMI}, the clustering performance is measured by \emph{clustering accuracy}, namely the number of correctly clustered samples over the total number of samples. %

The {ExtYaleB} dataset contains $2,414$ frontal face images of~38 subjects with different lighting, poses and illumination conditions, where each subject has round 64 faces. Following \cite{LRR2010}, we use $640$ faces from the first 10 subjects. Each face image is resized to $48\times42$ pixels and then reshaped as a $2016$-dimensional gray-level intensity feature. %
The HARUS dataset is a large dataset (containing 10,299 signals w.r.t.~6 activities) with data collected using embedded sensors on the smartphones carried by volunteers on their waists, %
when they are conducting daily activities (\emph{e.g.},~walking, sitting, laying). %
The captured sensor signals are pre-processed to filter noise and post-processed. %
Finally, a 561-dimensional feature vector with time and frequency domain variables is extracted for each signal. %

\def \ratiohere{0.48}
\begin{figure}[t]
\center
\subfigure[Running time (in log scale).]{\label{fig:TimeCompareLRR}\includegraphics[trim = 1mm 0mm 1mm
0mm, clip, width=\ratiohere\columnwidth]{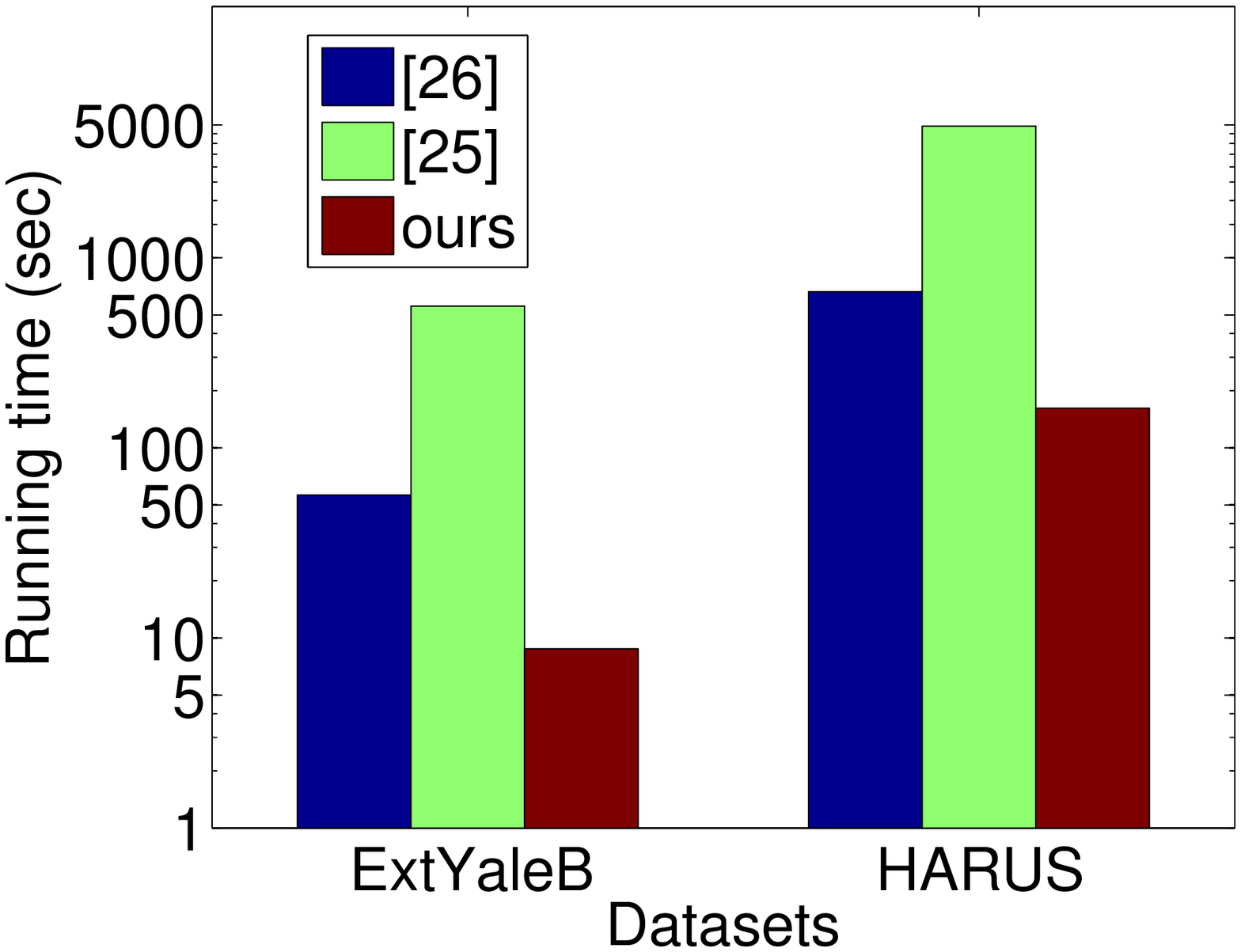}}
\subfigure[Clustering accuracy.]{\label{fig:AccCompareLRR}\includegraphics[trim = 1mm 0mm 1mm
0mm, clip, width=\ratiohere\columnwidth]{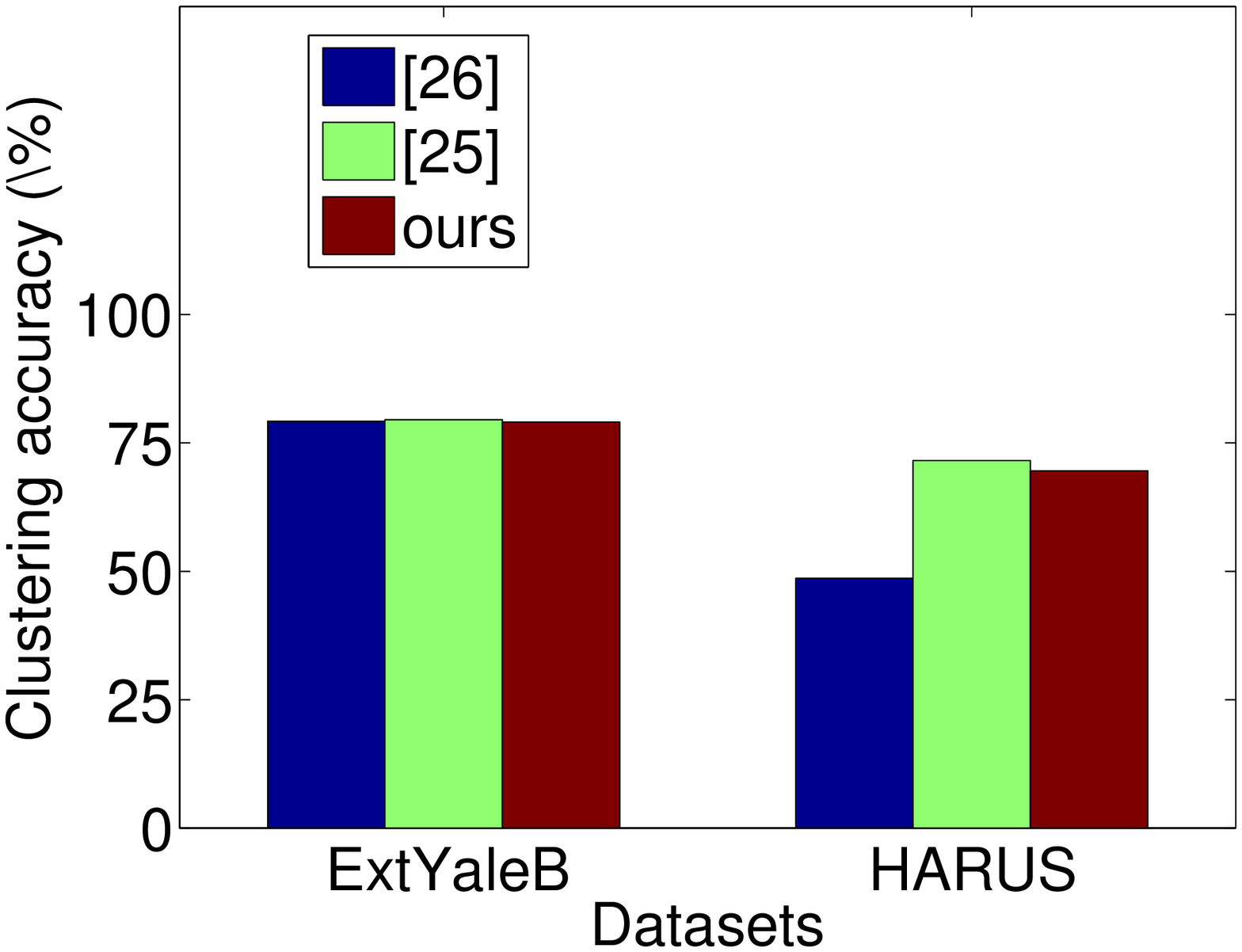}}
\caption{
\shijie{%
Comparing the running times and clustering accuracies of different LRR solvers on two datasets.%
}
}
\end{figure}
The best clustering accuracies and the corresponding running times are reported in %
\shijie{Figure~\ref{fig:AccCompareLRR}~and~Figure~\ref{fig:TimeCompareLRR}, respectively.} %
It can observed that, %
\shijie{our SP-RPRG method} %
outperforms the two existing LRR solvers in terms of efficiency, %
since our algorithm does not \shijie{frequently involve} SVDs~w.r.t. large matrices. %
Moreover, our algorithm achieves comparable clustering performance with {\cite{lin2011linearized}}. %
In contrast, the LRR solver in \cite{LRRPAMI} achieves lower clustering accuracy on the HARUS dataset, %
possibly because that algorithm is not  guaranteed to obtain a globally optimal solution. %

\section{Conclusion}
Classical proximal methods may require many large-rank SVDs when addressing the trace-norm regularized problems on vector spaces. To overcome this, we first propose a proximal Riemannian gradient (PRG)  method to address trace-norm regularized problems over a matrix variety $\mMLr$, where $r$ is supposed to be known. By performing optimization on $\mMLr$, PRG does not require SVDs, thus can greatly reduce the computation cost. %
A robust version of PRG method has also been proposed to %
\shijienew{handle the outlier cases.} %
To address  general trace-norm regularized problems, a subspace pursuit strategy is proposed by iteratively activating a number of active subspaces. Extensive experiments on two classical trace-norm based tasks, namely low-rank matrix completion and LRR based clustering, demonstrate the superior efficiency of the proposed methods over other methods.

\small
\bibliographystyle{plain}
\end{document}